\documentclass[runningheads,orivec]{llncs}
\usepackage{graphicx}
\usepackage{enumitem}
\usepackage{algpseudocode,algorithm,algorithmicx}
\usepackage{mathtools}
\usepackage[table]{xcolor}
\usepackage{multirow}
\usepackage{multicol}
\usepackage[flushleft]{threeparttable}
\usepackage{pifont}
\usepackage{array}
\usepackage{makecell}
\usepackage{longtable}
\usepackage[titletoc,title]{appendix}
\usepackage{lscape}
\usepackage{datatool}
\usepackage{forloop}
\usepackage{array}
\usepackage{commath}
\usepackage{hyperref}
\usepackage{environ}
\usepackage{listings}
\usepackage{subcaption}
\usepackage{todonotes}
\usepackage{proof}
\usepackage[scientific-notation=true]{siunitx}
\newtheorem{proofpart}{Part}
\makeatletter
\@addtoreset{proofpart}{theorem}
\makeatother
\newcommand{\repeattheorem}[1]{%
	\begingroup
	\renewcommand{\thetheorem}{\ref{#1}}%
	\expandafter\expandafter\expandafter\theorem
	\csname reptheorem@#1\endcsname
	\endtheorem
	\endgroup
}

\NewEnviron{reptheorem}[1]{%
	\global\expandafter\xdef\csname reptheorem@#1\endcsname{%
		\unexpanded\expandafter{\BODY}%
	}%
	\expandafter\theorem\BODY\unskip\label{#1}\endtheorem
}

\setlist[itemize]{noitemsep, topsep=1mm}

\algnewcommand\algorithmicinput{\textbf{Input:}}
\algnewcommand\Input{\item[\algorithmicinput]}

\newcommand{\WRP}{\par\qquad\(\hookrightarrow\)\enspace}

\begin{document}
\title{From Checking to Inference: Actual Causality Computations as Optimization Problems\thanks{Extended version with proofs of ATVA2020 paper with the same title. Work supported by the German Research Foundation (DFG) under grant no. PR1266/4-1, Conflict resolution and causal inference with integrated socio-technical models. }}
	
%
%the Deutsche Forschungsgemeinschaft (DFG) 
\titlerunning{Actual Causality Computations}
\ifdefined\anonymous
\author{Anonymized  author(s)}
\else
\author{ Amjad Ibrahim \and  Alexander Pretschner
	\institute{Department of Informatics, Technical University of Munich, Germany}
	\\\email{\{ibrahim,  pretschn\}@in.tum.de}
}
\fi
\authorrunning{A. Ibrahim and A. Pretschner}

%
\maketitle              % typeset the header of the contribution
\begin{abstract}
Actual causality is increasingly well understood. Recent formal approaches, proposed by Halpern and Pearl, have made this concept mature enough to be amenable to automated reasoning. Actual causality is especially vital for building accountable, explainable systems. Among other reasons, causality reasoning is computationally hard due to the requirements of counterfactuality and the minimality of causes. 
Previous approaches presented either inefficient or restricted, and domain-specific, solutions to the problem of automating causality reasoning.
In this paper, we present a novel approach to formulate different notions of causal reasoning, over binary acyclic models, as optimization problems, based on quantifiable notions within counterfactual computations. We contribute and compare \textit{two} compact, non-trivial, and sound integer linear programming (ILP) and  Maximum Satisfiability (MaxSAT) encodings to \textit{check} causality. Given a candidate cause, both approaches identify what a minimal cause is. Also, we present an ILP encoding to \textit{infer} causality without requiring a candidate cause. We show that both notions are efficiently automated. Using models with more than $8000$ variables,  checking is computed in a matter of seconds, with MaxSAT outperforming ILP in many cases. In contrast, inference is computed in a matter of  minutes.

%\begin{keyword}
%	actual causality\sep automated reasoning and inference 
%\end{keyword}
\end{abstract}
\section{Introduction}\label{sec:introduction}
%      cf when analyzing failures of critical systems
Actual causality is the retrospective linking of \textit{effects} to  \textit{causes}~\cite{halpern2016actual,pearl1996causation}. 
As part of their cognition, humans reason about actual causality to explain particular past events, to control future events, or to attribute moral responsibility and legal liability~\cite{hopkins2002strategies}. Similar to humans, it is useful for systems in investigating  security protocols~\cite{kuennemann2018automated}, safety accidents~\cite{ladkin1998because}, software or  hardware models~\cite{beer2012explaining,leitner-fischer2013causality,chockler2008causes}, and database queries~\cite{meliou2010causality}. 
More importantly, actual causality is central for enabling social constructs such as accountability in Cyber-Physical systems~\cite{nfm,Kacianka_2019,ibrahim2020ECAI}, in information systems~\cite{feigenbaum2011accountability}, 
and explainability in artificial intelligence systems~\cite{miller2018explanation}.  

Attempts to formalize a precise definition of an \textit{actual cause} go back to the eighteenth century when Hume~\cite{hume1748An} introduced \textit{counterfactual reasoning}. Simply put, counterfactual reasoning concludes that event $A$ is a cause of event $B$ if $B$ does not occur if $A$ does not occur. However, this simple reasoning cannot be used with interdependent, multi-factorial, and complex causes~\cite{lewis1973causation}. Recently, Halpern and Pearl formalized HP--a seminal model-based definition of actual causality that addresses many of the challenges facing naive counterfactual reasoning~\cite{halpern2016actual}.

Because of its formal foundation, HP enables automated causality reasoning. We distinguish two notions of reasoning: \textit{checking} and \textit{inference}. 
\textit{Checking} refers to verifying if a candidate cause is an actual cause of an effect, i.e., answering the question ``is $\vec{X}$ a cause of $\varphi$?'' \textit{Inference} involves finding a cause without any candidates, i.e., answering the question ``why $\varphi$?'' Using HP, causality checking is, in general, $D^P_1$-complete and $\mathit{NP}$-complete for singleton (one-event) causes~\cite{halpern2015modification}; the difference is due to a minimality requirement in the definition (details in \autoref{sec:preliminaries}). Intuitively, \textit{inference} is at least as hard. The complexity led to restricted (e.g.,  singleton causes, single-equation models \cite{meliou2010causality}) utilizations of HP (\autoref{sec:realted}). 
%To the best of our knowledge, 
All these utilizations exploit domain-specificities (e.g., database repairs \cite{meliou2010causality,salimi2014causes} \cite{bertossi2018characterizing}), which hinders taking advantage of the available approximations for general queries. In prior work, we proposed an approach to \textit{check} causality in acyclic models \textit{with binary variables} based on the satisfiability problem (SAT)~\cite{ibrahim2019}. 
The approach required enumerating all the satisfying assignments of a formula (ALL-SAT), which obviously is impacted by the solver's performance~\cite{zhao2009asig}. 
Thus, previous approaches fail to automate answering queries for larger models.

Large models of causal factors are likely to occur especially when generated automatically from other sources for purposes of accountability and explainability~\cite{ibrahim2020ECAI,nfm,miller2018explanation}. Further,  models of real-world accidents are sufficiently large to require efficient approaches.  For instance, a model of the 2002 mid-air collision in Germany consists of 95 factor~\cite{ueberWBG} (discussed in \cite{ibrahim2020ECAI}).  
Thus, in this paper, we present a novel approach to formulate actual causality computations in binary models as optimization problems. We show how to construct quantifiable notions within counterfactual computations, and use them for checking and inference. 

We encode our checking approach as integer linear programs (ILP), or weighted MaxSAT formulae~\cite{li2009maxsat}. 
Both are well-suited alternatives for Boolean optimization problems. However, MaxSAT has an inherent advantage with binary propositional constraints~\cite{li2009maxsat}. On the other hand, ILP has an expressive objective language that allows us to tackle the problem of causality inference as a multi-objective program.  Accordingly, we contribute an approach with \textit{three} encodings. The first two cover causality \textit{checking}, and better they can determine a \textit{minimal} HP cause from a potentially non-minimal candidate cause; we refer to this ability as \textit{semi-inference}.
The third encoding tackles causality \textit{inference}. 
All these encodings benefit from the rapid development in solving complex and large (tens of thousands of variables and constraints) optimization problems~\cite{koch2013progress,bacchus2018maxsat}. 

We consider our work to be the first to provide an efficient solution to the problem of checking and inferring HP causality,  for a large class of models (binary models) without any dependency on domain-specific technologies. We contribute: \textbf{1.)} A sound formulation of causality computations for acyclic binary models as optimization problems, \textbf{2.)} A Java library\footnote{\small\url{https://github.com/amjadKhalifah/HP2SAT1.0/tree/hp-optimization-library}} 
that implements the approaches.  %
\textbf{3.)} An empirical evaluation, using models from multiple domains, of the efficiency and scalability of the approaches in comparison with previous work.

\section{Halpern-Pearl Definition of Actual Causality}\label{sec:preliminaries}
HP uses variables to describe the world, and \textit{structural equations} to define its mechanics~\cite{pearl1996causation}. The variables are split into \textit{exogenous} and \textit{endogenous}. The values of the former, called a \textit{context $\vec{u}$}, are governed by factors that are not part of the modeled world (they represent the environment). The endogenous variables, in contrast, are determined by equations of exogenous and endogenous variables. In this formulation, we look at causes within a specified universe of discourse represented by the endogenous variables, while exogenous variables are not considered to be part of a cause but rather as given information. An equation represents the semantics of the dependency of the endogenous variable on other variables. Similar to Halpern, we limit ourselves to acyclic models in which we can compute a unique solution for the equations given a context $\vec{u}$, which we refer to as \textit{actual evaluation} of the model. %Causal models are visualized in networks where a node is a variable, and an edge is a dependency between two variables.
  A binary model (Boolean variables only) is formalized in \autoref{def:causal_model}.
\begin{definition}\label{def:causal_model} \cite{pearl1996causation}
	\textbf{Binary Causal Model}\\
	A causal model is a tuple $M = \mathcal{(U,V,R,F)}$, where
		\\- $\mathcal{U}$, $\mathcal{V}$ are sets of exogenous variables and endogenous variables respectively, 
		\\- $\mathcal{R}$ associates with $Y \in \mathcal{U \cup V}$ a set  of possible values $\mathcal{R}(Y)$, i.e., $\{0,1\}$,
		\\- $\mathcal{F}$ maps $X \in \mathcal{V}$ to a function $F_X : (\times_{U \in \mathcal{U}}\mathcal{R}(U)) \times (\times_{Y \in \mathcal{V}\backslash\{X\}}\mathcal{R}(Y)) \to \{0,1\}$
\end{definition}

\autoref{def:causal_model} makes precise the fact that $F_X$ determines the value of $X$, given the values of all the other variables. We summarize the causality notations before defining the cause in \autoref{def:ac}. A \textit{primitive event} is a formula of the form $X = x$, for $X \in \mathcal{V}$ and $x$ is a value $\in \{0,1\}$. A sequence of variables $X_1,...,X_n$ is abbreviated as $\vec{X}$. Analogously, $X_1=x_1,...,X_n=x_n$ is abbreviated $\vec{X}=\vec{x}$. $\varphi$ is a Boolean combination of such events.   $(M, \vec{u}) \models X = x$ if the variable $X$ has value $x$ in the unique solution to the equations in $M$ given context $\vec{u}$. The value of  variable $Y$ can be overwritten by a value $y$ (known as an intervention) writing $Y \leftarrow y$ (analogously $\vec{Y} \leftarrow \vec{y}$ for vectors). Then, a causal formula is of the form $[Y_1 \leftarrow y_1, ..., Y_k \leftarrow y_k]\varphi$, where $Y_1, ..., Y_k$ are variables in $\mathcal{V}$ that make $\varphi$ hold when they are set to $y_1,.., y_k$.  We write $(M, \vec{u}) \models \varphi$ if the causal formula $\varphi$ is true in $M$ given $\vec{u}$. Lastly, $(M, \vec{u}) \models [\vec{Y} \leftarrow \vec{y}]\varphi$ holds if we replace variable equations'  in $\vec{Y}$ by equations of the form $Y=y$ denoted by $(M_{\vec{Y} = \vec{y}},\vec{u}) \models \varphi$~\cite{halpern2016actual}.
\begin{definition} \label{def:ac}
	\textbf{Actual Cause} \cite{halpern2015modification}\\ $\vec{X} = \vec{x}$
	is an actual cause of $\varphi$ in $(M,\vec{u})$ if the following
	three conditions hold:\\ \textbf{AC1.} $(M, \vec{u}) \models (\vec{X} = \vec{x})$ and $(M, \vec{u}) \models \varphi$.\\ \textbf{AC2.} There is a set $\vec{W}$ of variables in $\mathcal{V}$ and a setting $\vec{x}'$ of the variables in $\vec{X}$ such that if $(M, \vec{u}) \models \vec{W} = \vec{w}$, then 	$(M, \vec{u}) \models [\vec{X} \leftarrow \vec{x}', \vec{W} \leftarrow \vec{w}]\neg\varphi$.\\ \textbf{AC3.} $\overrightarrow{X}$ is minimal: no non-empty subset of $\vec{X}$ satisfies AC1 and AC2.
\end{definition}
AC1 checks that the cause $\vec{X}=\vec{x}$ and the effect $\varphi$ occurred within the actual evaluation of $M$ given context $\vec{u}$, i.e., the cause is sufficient for the occurrence of the effect. 
AC2 checks the counterfactual (necessary) relation between the cause and effect.  It holds if there exists a setting $\vec{x}'$ for the cause variables $\vec{X}$ different from the actual evaluation $\vec{x}$ (in binary models such a setting is the negation of the actual setting \cite{ibrahim2019}), and another set of variables $\vec{W}$, referred to as a contingency set, that we use to \textit{fix} variables at their actual values, such that $\varphi$ does not occur. The contingency set $\vec{W}$ is meant to deal with issues such as preemption and redundancy. 
Preemption is a problematic situation where multiple possible causes coincide (illustrated by an example below)~\cite{lewis1973causation}; thus a naive counterfactual check cannot determine the  cause \cite{leitner-fischer2013causality}.  AC3 checks that $\vec{X}$ is minimal in fulfilling the previous conditions. 
To check a cause, we need to think of two worlds (variable assignments): the \textit{actual} world with all the  values known to us, and the counterfactual one in which the \textit{cause} and \textit{effect} take on different  values. Two factors further complicate the search for this counterfactual world. First, finding an arbitrary $\vec{W}$, such that AC2 holds which is exponential in the worst case. 
Second, no (non-empty) subset of $\vec{X}$ is sufficient for constructing such a counterfactual world.  
Halpern shows that checking causality is in general $D^P_1$-complete \cite{halpern2015modification,aleksandrowicz2014computational}, i.e., checking \textit{AC1} is $P$, checking \textit{AC2} is $NP$-complete, and checking \textit{AC3} is $co-NP$-complete.  Complexity considerations for binary models suggest a reduction to SAT or ILP \cite{halpern2015modification}; in this paper, we show concrete ILP and MaxSAT formulations to check, and an ILP formulation to infer a cause.

\textit{\textbf{Example}}: \textit{Throwing rocks}~\cite{lewis1973causation} is a problematic example from philosophy: Suzy and Billy both throw a rock at a bottle that shatters if one of them hits. We know Suzy's rock hits the bottle slightly earlier than Billy's and both are accurate throwers. Halpern models this story using the endogenous variables
$ST,$ $BT$ for ``Suzy/Billy throws'', with values 0 (the person does not throw) and 1 (s/he does), $SH, BH$ for ``Suzy/Billy hits'', and $BS$ for ``bottle shatters.'' Two exogenous variables $ST_{exo}, BT_{exo}$ are used to set the values. The equations:

\noindent -- $ST = ST_{exo}$\indent\indent\indent -- $BT = BT_{exo}$ \indent\indent\indent -- $SH = ST$ \\
\noindent -- $BH = BT \land \neg SH$ \indent -- $BS = SH \lor BH$

%\noindent\begin{minipage}[c]{0.45\textwidth}
%		\begin{itemize}
%		\item $BS = SH \lor BH$
%		\item $SH = ST$
%		\item $BH = BT \land \neg SH$
%		\item $ST = ST_{exo}$
%		\item $BT = BT_{exo}$\\
%	\end{itemize}
%\end{minipage}
%\hfill
%\begin{minipage}[c]{0.5\textwidth}
%	\begin{tikzpicture}
%\node (BT) at (0, 0) {$BT$};
%\node (ST) at (0, 1) {$ST$};
%\node (BH) at (1.5, 0) {$BH$};
%\node (SH) at (1.5, 1) {$SH$};
%\node (BS) at (3, 0.5) {$BS$};
%
%\draw[->]
%(BT) edge (BH)
%(ST) edge (SH)
%(SH) edge (BH)
%(SH) edge (BS)
%(BH) edge (BS);
%\end{tikzpicture}
%
%\captionof{figure}{Rock-throwing example \cite{halpern2015modification}}
%\label{fig:hp:billy_suzy_causal_network}
%\end{minipage}
Assuming a context (exogenous variables' values) $\vec{U}=\vec{u}$: $ST_{exo}=1, BT_{exo} =1$ (both actually threw), the actual evaluation of the model is: $ST=1$, $BT=1$, $BH=0$, $SH=1$, and $BS=1$. Assume we want to \textit{check} whether $ST=1$ is a cause of $BS=1$, i.e., is Suzy's throw a cause for the bottle shattering?  Obviously, AC1 is fulfilled as both appear in the actual evaluation. As a candidate cause, we set  $ST=0$ (for binary models a candidate cause is negated to check counterfactuality; see Lemma 1 in \cite{ibrahim2019}). A first attempt with $\vec{W} = \emptyset$ shows that AC2 does \textit{not} hold. However, if we randomly let $\vec{W} = \{BH\}$, i.e., we replace the equation of $BH$ with $BH = 0$, then AC2 holds because  $(M, \vec{u}) \models [ST \leftarrow 0, BH \leftarrow 0] BS=0$, and AC3 automatically holds since the cause is a singleton. Thus, $ST=1$ is a cause of $BS=1$. Let us check if $ST = 1 \land BT = 1$ is a cause for $BS = 1$.  AC1 and AC2 hold (obviously if they both did not throw, the bottle would not shatter with a $\vec{W} = \emptyset$) but AC3 does not. As we saw earlier,  $ST=1$ alone satisfies AC2. Hence $ST = 1 \land BT = 1$ are not a cause.

 As opposed to the all-or-nothing treatment of causality,  Chockler and Halpern added (\cite{chockler2004responsibility}, modified in \cite{halpern2016actual}) a notion of \textit{responsibility} to a cause. They introduced a metric, degree of responsibility ($dr$), that “measures the minimal number of changes needed to make $\varphi$ counterfactually depend on $X$.”\footnote{Their idea is often motivated with an example of $11$ voters that can vote for Suzy or Billy. If Suzy wins $6$-$5$, we can show that each Suzy voter is a cause of her winning. If Suzy wins $11$-$0$, then each subset of size six of the voters is a cause. The authors argue that in $11$-$0$ scenario, “a voter feels less responsible” compared to $6$-$5$ situation.} \autoref{def:responsibility} shows a shortened version of $dr$ \cite{chockler2004responsibility,halpern2016actual}, which we use for causality inference in our work. 
\begin{definition}\label{def:responsibility}
	The degree of responsibility of $X = x$ w.r.t. a cause $\vec X=\vec x$ for $\varphi$, denoted $dr((\vec X=\vec x),(X = x), \varphi)$, is $0$ if $X = x$ is not in $\vec X=\vec x$; otherwise is $1/(|\vec{W}| + |\vec{X}|)$ given that $|\vec{W}|$ is the smallest set of variables that satisfies AC2.
\end{definition}

% This example illustrates the role  of $\vec{W}$ to address challenges of the counter-factual test. It also shows the complexity of checking causality.

% Next, we consider a real-world complicated causal story. 
%finding the set $W$, and the minimality check adds more complexity by implying checking AC2 again for the power-set of causes. 

%the responsibility of ST = 1 is 1/2, because we had ÑW = {BH}, and the responsibility of BT = 1 is 0, because we showed that Billy’s throw is not a cause according to the HP definition.

\section{Approach}\label{sec:approach}
%in f only W 
Given the triviality of AC1, we presented in prior work, a SAT-based approach to check causality, focusing on AC2~\cite{ibrahim2019}.
The contribution was in how AC2 is encoded into a formula $F$, so that an efficient conclusion of $\vec{W}$ without iterating over the power-set of all variables, is possible. 
Briefly, $F$ described a counterfactual world that
incorporated (1) $\neg\varphi$, (2) a context $\vec{u}$ (size $n$), (3) a setting $\vec{x}'$ for a candidate cause, $\vec{X}$, and (4) a method to infer $\vec{W}$, while maintaining the semantics of $M$. Because checking is done in hindsight, we have the actual evaluation of the variables. Thus, the first three requirements are represented using literals. The semantics of $M$, given by each function $F_{V_i}$  corresponding to $V_i$ (according to \autoref{def:causal_model}), is expressed using an equivalence operator between a variable and its function, i.e.,  $V_i \leftrightarrow F_{V_i}$. This is not done for the cause variables because they are represented by a negation.  To account for $\vec{W}$, we add a disjunction to the equivalence sub-formula with the positive or negative literal of $V_i$, according to its actual evaluation ($1$ or $0$). With this representation of each variable, we check if such a counterfactual world is satisfiable, and hence AC2 holds.

% We present Theorem~\ref{theorem:F} to stress the soundness of $F$ (proof is presented in the supplementary material).

By generalizing $F$, we can also check minimality (AC3). 
Assume we remove the restriction on the cause variables $\vec{X}$, of only be negated literals (allowing them to take on their original values also), and call the new formula $G$. Then $G$ might be satisfiable for the negated cause $\vec{X} = \vec{x'}$ as well as all the other combinations of the cause set. Analyzing all the satisfying assignments of $G$ (All-SAT), allows us to check minimality. Specifically, if we find an assignment such that at least one conjunct of  $\vec{X} = \vec{x}$ takes on a value that equals the one computed from its equation, it means that it is not a required part of the cause, and hence, the cause is not minimal. In many situations,  All-SAT is problematic and decreases the performance, especially if $G$ is satisfiable for a large number of assignments~\cite{zhao2009asig}. \autoref{eq:gbase} shows the  construction of $G$. Because $Y_i$ is the variable form, and $y_i$ is the value, we use $f(Y_i = y_i)$ to convert the variable to a positive or a negative literal, i.e., $Y$ or $\neg Y$. 
\begin{small}
	\begin{equation}
	\begin{aligned}
	G :=  \neg\varphi \wedge \bigwedge\limits_{i=1\ldots n} f(U_i=u_i) 	\wedge\bigwedge\limits_{i=1\ldots m, \not\exists j\bullet X_j=V_i} \left(V_i \leftrightarrow F_{V_i} \lor f(V_i=v_i)\right)
	\end{aligned}
	\label{eq:gbase}
	\end{equation} 
\end{small}

$F$ and $G$ aid us in \textit{checking} if a candidate cause $\vec{X}$ is a \textit{minimal, counterfactual} cause of $\varphi$. If it is not, we cannot use them to \textit{find} a minimal cause from within $\vec{X}$, i.e., \textit{semi-inference}. We, also, cannot use them to find a cause without requiring a candidate cause, i.e., \textit{inference}. To efficiently achieve such abilities, we present a novel formulation of causal queries as optimization problems. 
\subsection{Checking and Semi-inference Queries as Optimization Problems} \label{subsec:ilp} 
 In this section, we focus on the computation of the minimality requirement in causality checking. For that, we conceptualize a technique to check AC2 and AC3 as one problem (AC1 is explicitly checked solely). The result of solving this problem can then be interpreted to conclude AC2, $\vec{W}$, AC3, and, better, what is a minimal subset of the cause if AC3 is violated (semi-inference). To compare the efficiency, we formulate the problem as an integer program, and a MaxSAT formula. Both techniques solve the problem  based on an \textit{objective} function--- a function whose value is minimized or maximized among feasible alternatives.

%expalin distance
To quantify an \textit{objective} for a causal check, we introduce an \textit{integer} variable that we call the \textit{distance}.
Similar to the Hamming distance, it measures the difference between the cause values when $\varphi$ holds true, i.e., actual world, and when it holds false, i.e., the counterfactual world. As shown in \autoref{eq:distance}, it is computed by counting the cause variables whose values assigned by a solver ($x_i'$) is different from their value under the given context ($x_i$). As we shall see, the \textit{distance} is equivalent to the size of the (minimal) cause within our check of a possibly non-minimal cause. As such, the \textit{distance} must be greater than $0$, since a cause is non-empty, and less or equal to the size of $\vec{X}$ ($\ell$), i.e., $1 \le distance \le \ell$. 
\begin{equation}
distance = \sum_{i=1}^{\ell} d(i) \qquad s.t \qquad d(i) = 
\begin{cases}
1-x_i',& x_i=1\\
x_i',& x_i=0
\end{cases}
\label{eq:distance}
\end{equation}

According to  AC3, our \textit{objective function} is then to minimize the  distance; we encode a causality check as an optimization problem that  \textit{minimizes} the number of cause variables while satisfying the constraints for AC2 (counterfactuality and $\vec{W}$). In the following, we present how to derive these constraints for the ILP formulation, and the MaxSAT encoding. Then, we discuss how to interpret the results to (semi-)infer a minimal cause from a possibly non-minimal cause. 

%Specific ILP
\textbf{ILP} is an optimization program with integer variables and linear constraints and objectives. To formulate such a program, we need three elements: \textit{decision variables,  constraints, and objective(s)}. Our \textit{decision variables} are, in addition to the \textit{distance},  the set of exogenous and endogenous variables from the model, i.e., $\vec{U} \cup \vec{V}$. Since we only consider binary variables, their values are bound to be $0$ or $1$.
Since ILP and SAT solvers can be used as complementary tools,  the translation from SAT to ILP is standard \cite{li2004satisfiability}. Therefore, we reuse formula $G$ (\autoref{eq:gbase}) to create the \textit{constraints}.  Constraints from $G$ contain the \textbf{a.)} effect not holding true, \textbf{b.)} the context, \textbf{c.)} each endogenous variable either follows the model equation or the actual value, i.e., part of the  set $\vec{W}$, \textbf{d.)} each element in the cause set $\vec{X}=\vec{x}$ is not constrained, i.e., its equation is removed. Transforming these constraints (on the Conjunctive Normal Form (CNF) level) into linear inequalities is straightforward; we have clauses that can be reduced to ILP directly, e.g., express $y=x_1 \lor x_2$ as $1 \geq 2*y-x_1-x_2 \geq 0$~\cite{li2004satisfiability}. In addition,  we add a constraint to calculate the distance according to \autoref{eq:distance}. 

\textbf{MaxSAT.}
The maximum satisfiability problem (MaxSAT) is an optimization variant of SAT~\cite{li2009maxsat}. In contrast to SAT, which aims to find a satisfying assignment of all the clauses in a  formula, MaxSAT aims to find an assignment that maximizes the number of satisfied clauses. Thus, MaxSAT allows the potential that some clauses are unsatisfied. In this paper, we use \textit{partial} MaxSAT solving, which allows specific clauses to be unsatisfied, referred to as \textit{soft clauses}; contrary to the \textit{hard} clauses that must be satisfied~\cite{li2009maxsat}. A soft clause can be assigned a \textit{weight} to represent the cost of not satisfying it. In essence, a weighted partial MaxSAT problem is a minimization problem that minimizes the cost over all solutions. Unlike ILP, the objective in MaxSAT is immutable. Thus, we need to construct our formula in a way that mimics the concept of the distance.

 %MaxSat is an alternative of 0-1 ILP. 
 %However, the performance of state of the art solvers are remarkably enhancing as can be seen in the MaxSat competition \cite{bibid} especially for the "very boolean" problems. 

%maxsat specific
As shown in \autoref{eq:gmax}, the MaxSAT encoding also uses $G$ (shown in \autoref{eq:gbase}) as a base. $G$ embeds all the mandatory parts of any solution. Thus, we use the CNF clauses of $G$ as \textit{hard} clauses. On the other hand, we need to append the cause variables ($\vec{X}$) as \textit{soft} clauses (underlined in \autoref{eq:gmax}). Since the solver would minimize the cost of unsatisfying the ($\vec{X}$) clauses, we represent each cause variable as a literal according to its original value (when $\varphi$ holds). Because this is already in CNF, it is easier to assign weights. We assign $1$ as a cost for unsatisfying each cause variable's clause, i.e., when $X_i$ is negated in the (solved) counterfactual world. Then, the overall cost of unsatisfying the underlined parts of the formula is the count of the negated causes, i.e., the size of the minimal cause. Essentially, this concept maps directly to the \textit{distance}, which the MaxSAT solver will minimize. In contrast to ILP, we cannot specify a lower bound on the MaxSAT objective. 
Thus, we need to express the non-emptiness of a cause, as \textit{hard} clauses. A \textit{non-empty} cause means that at least one cause variable $X_j$ does not take its original value, \textit{and} does not follow its equation due to an intervention. The first conjunction (after $G$) in \autoref{eq:gmax} ensures the first requirement, while the second corresponds to the second case.
% The first ensures that not all the variables are equal to their original value. The second ensures that not every variable is allowed to evaluate according to its equation.   
\begin{small}
\begin{equation}
\begin{aligned}
G_{max} :=  G \wedge \neg(\bigwedge\limits_{i=1\ldots\ell}  f(X_i=x_i)) \land \neg(\bigwedge\limits_{i=1\ldots\ell}  X_i \leftrightarrow F_{X_i}) 	 \underline{\wedge\bigwedge\limits_{i=1\ldots\ell} f(X_i= x_i)}
\end{aligned}
\label{eq:gmax}
\end{equation} 
\end{small}
\noindent\textbf{Results.} 
With the above, we illustrated the formulation of a causal checking problem. We now discuss how to translate their results to a causal answer once they are solved; Algorithm~\ref{algorithm:ILP} formalizes this. The evaluation, in the input, is a list of the variables in $M$ and their values under $\vec{u}$.  Assuming $\vec{C}$ is a representation of the optimization problem (a set of linear constraints (without the objective), or hard/soft clauses), then in Line \ref{alg:line:ilp:assignment}, we solve this problem and process the results in Lines \ref{alg:line:ilp:x}-\ref{alg:line:ilp:w}. The feasibility (satisfiability) of the problem implies that either $\vec{X}$ or a non-empty subset of it is a minimal cause (fulfills AC2 and  AC3). If \textit{distance} (cost returned by the MaxSAT solver) equals the size of $\vec{X}$, then the whole candidate cause is minimal. Otherwise, to find a minimal cause $\vec{X}_{min}$ (semi-inference), we choose the parts of $\vec{X}$ that have different values between the actual and the solved values (Line~\ref{alg:line:ilp:x}).  To determine $\vec{W}$, in Line~\ref{alg:line:ilp:w}, we take the variables whose solved values are the same as the actual evaluation (potentially including $\vec{X}$ variables). Obviously, this is not a minimal $\vec{W}$, which is not a requirement for \textit{checking} HP~\cite{ibrahim2019}. If the model is \textit{infeasible or unsatisfiable}, then HP for the given $\vec{X}$ (checking) and its subsets (semi-inference) does not hold.  
%Note that an optimization problem can have more than one optimal solution. This implies that multiple minimal causes are found. We collect them all, but we present the first one. To pick one of the causes, we need to employ additional metrics.  
\begin{algorithm} [t]
	\caption{Interpreting the Optimization Problem's Results}\label{algorithm:ILP}
	\begin{algorithmic}[1]
		\Input causal model $M$, context $\langle U_1,\ldots,U_n\rangle=\langle u_1,\ldots,u_n\rangle$, effect $\varphi$, candidate cause $\langle X_1,\ldots,X_\ell\rangle = \langle x_1,\ldots,x_\ell\rangle$, evaluation $\langle V_1,\ldots,V_m\rangle=\langle v_1,\ldots,v_m\rangle$
		\Function{CheckCause}{$M, \vec{U}=\vec{u}, \varphi, \vec{X} = \vec{x},\vec{V}=\vec{v}$}
  		\If{$\langle U_1=u_1\ldots U_n=u_n,V_1=v_1'\ldots V_m=v_m'\rangle =$  $\textit{solve}(\vec{C},\textit{objective})$}\label{alg:line:ilp:assignment}
  		\State $\vec{X}_{min} := \langle X'_1...X'_d\rangle$ s.t. $\forall i\forall j\bullet (i\not= j\Rightarrow$ $X'_i\not= X'_j) \wedge (X'_i=V_j\Leftrightarrow v_j'\not= v_j)$\label{alg:line:ilp:x}
		\State $\vec{W} := \langle W_1...W_s\rangle$ s.t. $\forall i\forall j\bullet (i\not= j\Rightarrow$ $W_i\not= W_j) \wedge (W_i=V_j\Leftrightarrow v_j'=v_j)$\label{alg:line:ilp:w}
		\State \Return{$\vec{X}_{min},\vec{W}$}
		\Else{}
		\Return{\textit{infeasible (unsatisfiable)}}
		\EndIf
		\EndFunction
	\end{algorithmic}
\end{algorithm}

\noindent\textbf{Throwing rocks Example:} To illustrate our approach, we show the ILP and MaxSAT encodings to answer the query is $ST=1, BT=1$ a cause of $BS=1$?
\begin{small}
	\begin{equation*}
	\begin{aligned}
	min \enspace d\enspace \textrm{s.t.}\enspace& 
	\{BS  =  0,   \enspace ST_{exo} =  1, \enspace BT_{exo} =  1,  \enspace - SH + BS  \geq  0, \enspace - BH + BS \geq  0,  \\\enspace& - ST + SH  \geq  0, \enspace BT - BH  \geq  0, \enspace - SH - BH  \geq  -1,
	\enspace ST + BT + d  =  2  \}\\
	 G_{max}\enspace\enspace = \enspace&\neg BS  \;\land\; ST_{exo} \land BT_{exo} \;\land\; (BS \leftrightarrow SH \lor BH)
	\;\land\; ((SH \leftrightarrow ST) \lor SH) \;\land \\ 
	\enspace&((BH \leftrightarrow BT \land \neg SH) \lor\neg BH)
	\;\land\;\neg (ST \land BT ) \;\land\; \underline{(ST \land BT)}
	\end{aligned}
	\end{equation*}
\end{small}
Both encodings are solved with a $distance (d)$ (cost) value of  $1$, which indicates that $ST,BT$ is not minimal, and a cause of size $1$ is (semi)-inferred, namely $ST$. The optimal assignment ($\neg BS, ST_{exo}, BT_{exo}, \neg SH, \neg BH,  \neg ST, BT$) showed that the constraints can be guaranteed without changing the value of $BT$, which violates AC3. This shows the enhancement of finding a minimal cause rather than only checking AC3. \autoref{theorem:ilp} states the soundness of our approach  (for proofs see \autoref{sec:ac3_proof})
\begin{reptheorem}{ilp}\label{theorem:ilp}
 The generated optimization problem (ILP program or $G_{max}$) is feasible iff AC3 holds for $\vec{X}$ or a non-empty subset of $\vec{X}$. 
\end{reptheorem}
\subsection{Causality Inference with ILP} \label{sec:why}
The previous approaches utilized the candidate cause $\vec{X}$ to help describe a counterfactual world that proves $\vec{X}$ is a cause of $\varphi$. In this section, we present a method, $ILP_{why}$, to \textit{infer} causality (answer \textit{why $\varphi$?} questions) without requiring $\vec{X}$. Unlike checking, in inference, we cannot aid the solver in a description of the counterfactual world (e.g., negating values of $\vec{X}$). Instead, we describe characteristics of the actual cause that have caused an effect $\varphi$.

%res coz we are looking from scratch so we need more characteristic 
In addition to requirements of counterfactuality and minimality imposed by the conditions in \autoref{def:ac}, we utilize the \textit{degree of responsibility} ($dr$) as a mean to compare actual causes~\cite{chockler2004responsibility}. While the conditions are suitable for determining if $\vec{X}$ is a cause, $dr$ judges the ``quality" of the cause based on an aggregation of its characteristics. Because we may find multiple causes for which the conditions hold, $dr$ is reasonable for comparison. We require our answer to an inference question to be an actual cause with the maximum $dr$.  We come back to this after we construct a formula $G^*$ that is the base of $ILP_{why}$. %This seemed natural, given our previous approaches.

% the new concept is to use a formula again; then add auxilairy vars and ojectives; do i use the distance?
Both negating the effect formula ($\neg\varphi$) and setting the context $f(U_i=u_i)$ remain as in \autoref{eq:gbase}. Because the variables that appear in the effect formula cannot be part of the cause, we represent each with the simple \textit{equivalence} relation, i.e., $V_i \leftrightarrow F_{V_i}$. The complicated part is representing the other variables because any variable can be: \textbf{a.} a cause, \textbf{b.} a contingency-set, or \textbf{c.} a normal variable. Recall, in a counterfactual computation, a cause does not follow its equation, \textit{and} differs from its original value; a contingency-set variable does not follow its equation while keeping its original value; a normal variable follows its equation, regardless of whether it equals the original value or not. Thus, we need to allow variables to be classified in any category in the ``best" possible way. 

To that end, we represent each (non-effect) variable $V_i$ with a disjunction between the equivalence holding and not holding, \textit{and} a disjunction between its original value and its negation: $ \left(  \left(V_i \leftrightarrow F_{V_i}\right)   \lor  \neg \left(V_i\leftrightarrow F_{V_i}\right)  \right) 
		\wedge \left(   V_{i_{orig}} \lor \neg V_{i_{orig}} \right)  $.
Clearly, each disjunction is a tautology. However, this redundancy facilitates the classification into the categories; more importantly, we can incentivize the solver to classify those variables according to specific criteria. 

To be able to guide the solver, we add auxiliary boolean variables (indicators) to each clause (left and right parts of a disjunction). They serve two functions. The first is to \textit{indicate} which clauses hold. Since the two parts of the conjunction are not mutually exclusive, i.e., a variable can follow its equation, yet have its original value, we need \textit{two} indicators $C^1 C^2$. Secondly, similar to the concept of \textit{distance} from \autoref{subsec:ilp}, we use the indicators to describe criteria of the solution. For each variable $V_i$, $C^1_{i}$ is appended to the first two clauses:
$\left( \left(V_i \leftrightarrow F_{V_i}\right) \land C^1_{i} \right)  \lor \left( \neg \left(V_i \leftrightarrow F_{V_i}\right) \land \neg C^1_{i} \right)$. Similarly  $C^2$ is appended to the other clauses. As such, the category of each endogenous variable is determined based on values of $C^1$ and $C^2$. A cause variable would have a $C^1 C^2: 00$ (not following the formula nor its original value); a contingency-set variable has a $C^1 C^2: 01$; and a normal variable has a $C^1 C^2: 10$, or $11$. Formula $G^*$ follows (equivalence relations of effect variables are omitted for space). 
\begin{small}
\begin{equation*}
	\begin{aligned}
		G^* :=  & \neg\varphi \wedge \bigwedge_{i=1\ldots n} f(U_i=u_i) 
		\wedge\bigwedge_{i=1\ldots m} \left( \left( \left(V_i \leftrightarrow F_{V_i}\right) \land C^1_{i} \right)  \lor \left( \neg \left(V_i \leftrightarrow F_{V_i}\right) \land \neg C^1_{i} \right) \right) \\ 
		&\land \left( \left( V_{orig} \land C^2_{i} \right)  \lor \left( \neg V_{orig} \land \neg C^2_{i} \right) \right) 
	\end{aligned}
	\label{eq:gstar}
\end{equation*} 
\end{small}

\begin{reptheorem}{why}\label{theorem:why}
	Formula $G^*$ is satisfiable iff $\exists$ $\vec{X}=\vec{x}$ such that AC2 holds for $\vec{X}$ 
\end{reptheorem}

We now discuss the \textit{objectives} of this formulation. We aim to find an assignment to the constraints in $G^*$ that corresponds to a cause with a maximum $dr$. Recall that $dr$ is $1 / (|\vec{X}|+|\vec{W}|)$. \textit{Maximizing} $dr$ entails \textit{minimizing} $|X|+|W|$. Since the  three sets (cause, contingency, and normal) form the overall model size (excluding  effect and exogenous variables), then minimizing $|X|+|W|$ is equivalent to maximizing the number of normal variables, which concludes our \textit{first objective}. The sum of $C^1$ variables  resembles the number of normal variables; thus, \textit{$objective_1$} is to \textit{maximize} the sum of $C^1$ variables.

The above formulation minimizes $\vec{X}$, and $\vec{W}$ as a whole, following $dr$. For our purpose, we think it is valid to look for causes with higher responsibility first (fewer variables to negate or fix) and favor them over smaller causes. For example, if an effect has two actual causes: one with $2$ variables in $\vec{X}$, $3$ in $\vec{W}$, and the second with $1$ variable in $\vec{X}$, $5$ in $\vec{W}$, we pick the first. That said, we still want to distinguish between $\vec{X}$ and $\vec{W}$ in causes with the same $dr$. Assume we have two causes: the first with $2$ variables in $\vec{X}$,  $3$ in $\vec{W}$, and the second with $3$ in $\vec{X}$, $2$ in $\vec{W}$. Although both are optimal solutions to \textit{$objective_1$}, we would like to pick the one with fewer causes. Thus, we add \textit{\textit{$objective_2$}} to \textit{minimize}   causes, i.e., the number of variables with $C^1$ and $C^2$ equal to $0$. We use hierarchical objectives in ILP, for which the solver finds optimal solution(s) based on the first objective, and then use the second objective to optimize the solution(s). 
  %The second objective required a third control variable, i.e., $C_3$ that captures the case when the first two control variables are $0$. In other words, $C_3= \neg C_1 \wedge \neg C_2$.  

We wrap-up with Algorithm~\autoref{algorithm:why}, which omits the construction of $G^*$. We start by turning  $G^*$ into  linear constraints in Line~\autoref{alg:line:why:constraints}.  The first objective $\textit{obj}_1$, which  maximizes $dr$ by maximizing the sum of $C^1_i$ is added in Line~\autoref{alg:line:why:objective}. 
The second objective, $\textit{obj}_2$, handles minimizing the size of the cause set. 
\begin{algorithm} [t]
	\caption{Causality Inference using $ILP_{why}$}\label{algorithm:why}
	\begin{algorithmic}[1]
		\Input causal model $M$, context $\langle U_1,\ldots,U_n\rangle=\langle u_1,\ldots,u_n\rangle$, effect $\varphi$, evaluation $\langle V_1,\ldots,V_m\rangle=\langle v_1,\ldots,v_m\rangle$
		\Function{FindCause}{$M, \vec{U}=\vec{u}, \varphi, \vec{V}=\vec{v}$}
			\State{ \label{alg:line:why:constraints}
			$\langle Con_1,\ldots Con_n\rangle =$  $\textit{convertToILP}(\textit{CNF}(G^*))$	}
		\State{ \label{alg:line:why:objective}
			$\textit{obj}_1 = Maximize\;  \sum_{i=1}^{m} C^1_{i}\;  s.t.\;  \textit{obj}_1 \leq |\vec{V}|$} 
		\State{ \label{alg:line:why:objective2}
			$\textit{obj}_2 = Minimize\; \sum_{i=1}^{m} (1 - C^1_{i}) * (1- C^2_{i})\;  s.t.\;   |\vec{V}|  \geq \textit{obj}_2 \geq 1$}
		\If{$\langle V_1=v_1'\ldots V_m=v_m', C^1_{1}= c^1_{1}\ldots C^1_{m}= c^1_{m}, C^2_{1}= c^2_{1}\ldots  C^2_{m}= c^2_{m} \rangle$ \WRP $  = \textit{solve}(\vec{Con},\textit{obj}_1, \textit{obj}_2)$}\label{alg:line:why:assignment}
		\State $\vec{X'} := \langle X'_1...X'_{obj_{2}}\rangle$ s.t. $\forall i\forall j\bullet (i\not= j\Rightarrow$ $X'_i\not= X'_j) \wedge (X'_i=V_j\Leftrightarrow \neg c^1_{j} \wedge \neg c^2_{j})$\label{alg:line:why:x}
		\State $\vec{W} := \langle W_1...W_s\rangle$ s.t. $\forall i\forall j\bullet (i\not= j\Rightarrow$ $W_i\not= W_j) \wedge (W_i=V_j\Leftrightarrow (\neg c^1_{j} \wedge  c^2_{j}))$\label{alg:line:why:w}
		\State \Return{$\vec{X'},\vec{W}$}
		\Else{}
		\Return{\textit{infeasible}}
		\EndIf
		\EndFunction
	\end{algorithmic}
\end{algorithm}
We process the results after solving the program in Lines~\autoref{alg:line:why:assignment}-\autoref{alg:line:why:w}. The feasibility of the program means we found a cause (size $\textit{obj}_2$) with the maximum $dr$ for the effect. For the details, we check the indicators of each variable. The cause is composed of variables that have $C^1$ and $C^2$ equal $0$; variables in $\vec{W}$, have $C^1=0$ and $C^2=1$.  

\noindent\textbf{Throwing rocks Example.} Assume we want to answer \textit{why did the bottle shatter $BS=1$? (given both threw)}. The generated program is not shown, but it was solved with ($obj_1 = 2$), i.e., \textit{two} normal variables, and $obj_2=1$, one cause variable.  Based on the indicators, $SH=1$ is the actual cause of $BS=1$, given that $BH=0$. This is the result of having   $C^1_{SH} = 0 \wedge C^2_{SH} =0 $ as opposed to $C^1_{BH} = 0 \wedge C^2_{BH} = 1$.  The result is correct; $SH$ is a cause of $BS$, with the maximum $dr$. Previous references of this example concluded $ST$ as a cause; however, since $SH$ is an identity function, this does not compromise our result.~\footnote{Arguably, the (geodesic) distance between the cause and effect nodes in the graph, can be taken into consideration. In this paper, we do not consider this issue}
\section{Evaluation}\label{sec:evaluation}
To evaluate their efficiency, we implemented our strategies as an open-source library. We used state of the art solvers: \textit{Gurobi} \cite{gurobi} for \texttt{ILP}, and Open-WBO for \texttt{MaxSAT}~\cite{wbo}. In this section, we evaluate the performance, in terms of  \textit{execution time} and \textit{memory allocation}, of the strategies in comparison with previous work. 

{\textbf{Experiment Setup}}
Unfortunately, there are no standard data-sets to benchmark causality computations. Thus, we gathered a dataset of $37$ models, which included $21$ \textit{small} models ($\le 400 $ endogenous variables)--from domains of causality, security, safety, and accident investigation-- and $16$ larger security models from an industrial partner, in addition to artificially generated models. The smaller models contained $9$ illustrative examples from literature (number of endogenous variables in brackets) such as $\mathit{Throwing-Rocks (5)}$, $\mathit{Railroad (4)}$~\cite{halpern2015modification}, $2$ variants of a safety model that describes a leakage in a subsea production system $\mathit{LSP (41)}$ and $\mathit{LSP2 (41)}$~\cite{cheliyan2018fuzzy}, and an aircraft accident model (Ueberlingen, 2002) $\mathit{Ueb (95)}$ \cite{ueberWBG}, $7$ generated binary tress, and a security model obtained from an industrial partner which depicts how insiders within a company steal a master encryption key $\mathit{SMK}$. Because it can be parameterized by the number of employees in a company, we have $14$ variants of $\mathit{SMK}$, $2$ small ones $\mathit{SMK1 (36)}$ and $\mathit{SMK8 (91)}$, and $12$ large models of sizes ($550-7150$).  In addition, we artificially generated $4$ models: $2$ binary trees with different heights, denoted as $\mathit{BT (2047 - 4095)}$, and $2$  trees combined with non-tree random models, denoted as $\mathit{ABT (4103)}$, and  $\mathit{ABT2 (8207)}$. 
We have evidence that such large models are likely  to occur when built automatically from architectures or inferred from other sources~\cite{ibrahim2020ECAI,nfm}. 
Details on the models and the results can be found online.\footnote{Machine-readable models and their description available at \url{https://git.io/Jf8iH}} 

We formulated a total of $\mathbf{484}$ \textit{checking} queries that vary in the context, cause, effect, and consequently differ in the result of AC1-AC3, the size of $\vec{W}$, and the size of the minimal cause. 
For the smaller models, we specified the queries manually according to their sources in literature, and verified that our results match the sources. The approaches, including previous ALL-SAT approach, answered these queries in under a second; hence, we exclude them from our discussion. For the larger models we constructed a total of $\mathbf{224}$ checking queries. We specified some effects (e.g., root of $\mathit{BT}$, or \textit{steal pass phrase} in $\mathit{SMK}$) and used different contexts, and randomly selected causes (sizes 1, 2, 3, 4, 10, 15, and 50) from the models. Since we can reuse the checking queries for inference by omitting the cause, we created $\mathbf{180}$ inference queries including $\mathbf{67}$ query of large models.

 We collected the results for: \texttt{SAT} - the original SAT-based  approach \cite{ibrahim2019}, and the presented three approaches: \texttt{ILP}, \texttt{MaxSAT}, and \texttt{ILP$_{\texttt{why}}$}- the inference approach. We ran each query for $30$ warm-ups (dry-runs before collecting results to avoid accounting for factors like JVM warm-up), and $30$ measurement iterations on an i7 Ubuntu machine with $16$ GB RAM. 
We set the cut-off threshold to $2$ hours. 
%overall number of is queries: 484; 224 more than a second(SAT did 187); why queries : 180 (67 large) finished 63. 

{\textbf{Discussion.}}
Generally, we use cactus plots to compare the performance of the approaches. The x-axis shows the number of queries an approach answered 
ordered by the execution time, which is shown on the y-axis; a point ($x,y$) on the plot reads as $x$ queries can be answered in $y$ or less. 
Next, we discuss the overall trends of the results; however, since we are interested in notions of checking, and inference, we also mention specific queries in which AC3 does not hold. 

As expected, the experiments confirmed the problems with the \texttt{SAT} encoding--- significant solver slow-down and memory exhaustion---\cite{zhao2009asig}. Thus, as shown in  \autoref{fig:sub1}, \texttt{SAT} only answered $\mathbf{187}$ of the $\mathbf{224}$ checking queries; for the remaining either it ran out of memory or took more than $2$ hours. For instance, queries on $\mathit{SMK (6600)}$ checking causes of sizes $2,3,4$ were not answered because the program ran out of memory. With almost all answered queries, \texttt{SAT} took al least two to four times as much as \texttt{ILP}, and up to twenty times as much as  \texttt{MaxSAT}.
In extreme cases, \texttt{SAT} took around $113$ minutes to finish, whereas others stayed under $5s$ for the same cases. Memory allocation, shown in \autoref{fig:sub2}, was similar to the execution time. However, it showed less difference with \texttt{ILP} and sometimes better allocation. Although it is not surprising that an ALL-SAT encoding performs poorly in some situations, the key result is that both \texttt{ILP} and \texttt{MaxSAT} provide more informative answers to a query while performing better. 

% should like into insights in the data ilp mxsat or from summary
According to our dataset, both \texttt{ILP} and \texttt{MaxSAT}, answered all queries in less than $70-100$ seconds. Especially for semi-inference, cases of \textit{non-minimal} causes and a minimal cause can be found, they are effective. For instance,  with queries using $\mathit{ABT (4103)}$, we found causes of size $2$, $5$, and $11$ out of candidate causes of sizes $5$, $10$, $15$, and $50$. 
All these queries were answered in around $5s$ using \texttt{ILP}, and $2s$ using \texttt{MaxSAT}. For larger and more complex models e.g., $\mathit{SMK (7150)}$, answering similar queries jumped to $98s$ with \texttt{ILP} and $71s$ \texttt{MaxSAT}. 

As shown in \autoref{fig:sub1} and \autoref{fig:sub2}, \texttt{MaxSAT} outperformed \texttt{ILP} in execution time and memory; a scatter plot to compare them is shown in \autoref{fig:sub3}. The propositional nature of the problem gives an advantage to \texttt{MaxSAT}. Especially for easier queries, as shown in \autoref{fig:sub3} bottom left, \texttt{MaxSAT} is much faster because no linear transformation is needed, which explains why the gap between the two decreases among the larger queries. Further, we used Open-WBO ---a solver that uses cores to initiate (UN)SAT instances \cite{wbo}--- which performs better, especially when the number of hard clauses is high~\cite{bacchus2014cores}. That said, in addition to the comparison, we used \texttt{ILP} for binary computations to incorporate quantifiable notions to infer causality  using multi-objective ILP in \texttt{ILP$_{\texttt{why}}$}.
\begin{figure}[t]
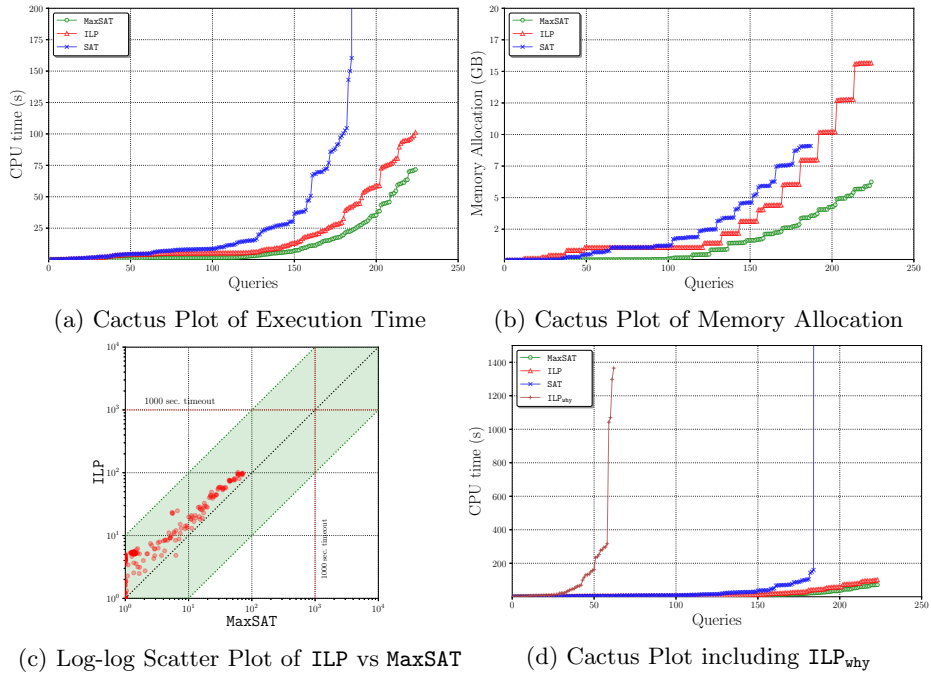

	\begin{subfigure}{.5\textwidth}
		\centering
		\scalebox{0.3}{\input{data/plotdata/plots/cactus/ct-is-short.pgf}}
		\caption{Cactus Plot of Execution Time}
		\label{fig:sub1}
	\end{subfigure}%
	\begin{subfigure}{.5\textwidth}
		\centering
			\scalebox{0.3}{\input{data/plotdata/plots/cactus/memory/cm-is-short.pgf}}
		\caption{Cactus Plot of Memory Allocation}
		\label{fig:sub2}
	\end{subfigure}
\newline
	\begin{subfigure}{.5\textwidth}
	\centering
	\scalebox{0.3}{\input{data/plotdata/plots/scatter/sall1.pgf}}
	\caption{Log-log Scatter Plot of \texttt{ILP} vs \texttt{MaxSAT}}
	\label{fig:sub3}
\end{subfigure}%
\begin{subfigure}{.5\textwidth}
	\centering
	\scalebox{0.3}{\input{data/plotdata/plots/cactus/ct-why-short.pgf}}
	\caption{Cactus Plot including  \texttt{ILP$_{\texttt{why}}$}}
	\label{fig:sub4}
\end{subfigure}
	\caption{Execution Time and Memory Results on the Larger Models}
	\label{fig:test}
	\vspace{-3mm}
\end{figure}

Although we have fewer inference queries ($67$), for comparison, we plot the checking approaches with \texttt{ILP$_{\texttt{why}}$} in \autoref{fig:sub4}.
\texttt{ILP$_{\texttt{why}}$} answered $\mathbf{63}$ out of $\mathbf{67}$ queries. In comparison, it was slower than the checking approaches. Still, it scaled to large and complex queries. For instance, with basic tree models of $4000$ variables ($\mathit{BT_{11}}$, $\mathit{ABT}$), it took $8s$, and scaled to $8000$ variable $\mathit{ABT2}$ within $63s$. However, it slowed down with larger models with complex semantics, i.e., $\mathit{SMK}$ different variants.  For instance, $\mathit{SMK (5500)}$ took $280s$, while  $\mathit{SMK (6600)}$ jumped to $1400s$. The slow down is related to the memory allocation because the program, finally, ran out of memory with queries on $\mathit{SMK (7150)}$. Given sufficient memory, we think  \texttt{ILP$_{\texttt{why}}$} computes \textit{inference} for even larger models.

In summary, we argue that the three approaches efficiently automates actual causality reasoning over binary models. Our \texttt{MaxSAT} encoding performs well for purposes of causality checking and semi-inference. Although slower, \texttt{ILP$_{\texttt{why}}$}  is also efficient and scalable for purposes of inference. 
\section{Related Work} \label{sec:realted}
There are three versions of HP (\textit{original} 2001, \textit{updated} 2005, \textit{modified} 2015 ) \cite{halpern2016actual}. We use the latest because it solves issues with the previous versions, and reduces the complexity \cite{halpern2015modification}.  To the best of our knowledge, no previous work tackled the implementation of the (\textit{modified}) HP. Still, we discuss the implementations of previous versions. Previous work has proposed simplified adaptations of the definition for various applications. First, in the domain of databases  \cite{meliou2010causality,bertossi2018characterizing,salimi2014causes}, (updated) HP  was  utilized to explain conjunctive query results.  The approaches heavily depend on the correspondence between causes and domain-specific concepts such as lineage, database repairs, and denial constraints. The simplification in that line of work is the limitation to a single-equation causal model based on the lineage of the query in \cite{meliou2010causality}, or no-equation model in \cite{bertossi2018characterizing,salimi2014causes}, in addition to the elimination of preemption treatment.  
Similar simplification has been made for Boolean circuits in \cite{chockler2008causes}. Second, in the context of software and hardware verification, (updated) HP is used to explain counterexamples returned by a model checker \cite{beer2012explaining}. The authors also restricted the definition to singleton causes and no-equation models. Third, in \cite{beer2015symbolic,leitner-fischer2013causality}, the authors adapted HP to debug models of safety-critical systems.  
Similar to our approach, all the papers above use acyclic binary models. However, they depend heavily on the correspondence between causes and domain-specific concepts.  Also, for efficiency, they relax the definition by restricting the model, i.e., one equation\cite{meliou2010causality}, no-equations  \cite{bertossi2018characterizing,salimi2014causes,beer2012explaining}, 
or by restricting the cause, i.e., singleton~\cite{beer2012explaining,salimi2014causes}; the complexity is then relaxed, because AC2 is straightforward (no $\vec{W}$) or AC3 is not needed. While such limitations are sufficient for the particular use-case, we  argue that they cannot be used outside their respective domains, e.g., for accountability. In contrast, our approach is a general method to automate HP. We focus on the minimality, which, to the best of our knowledge, no previous work has tackled. We employ optimization solving, which was not utilized before in this context.  Alternatively, previous work used SAT directly \cite{ibrahim2019},  indirectly \cite{beer2015symbolic}, or answer set programming \cite{bertossi2018characterizing}.
Sharing our generality, Hopkins proposed methods to check (original) HP using search-based algorithms~\cite{hopkins2002strategies}. Our approach scales to thousands of variables, while the results presented in the search-based approaches showed a limit of 30 variables. 

Fault tree analysis (FTA) is an established design-time method to analyze safety risks of a system \cite{ruijters2015fault}. FTA's primary analysis is the computation of  minimal cut sets MCSs of a fault tee; a CS is a set of events that, together, cause the top-level event. 
Approaches to determine MCS use Boolean manipulation, or Binary decision diagrams \cite{ruijters2015fault}. These methods are similar to our computations; however, the conceptual difference is the definition of a cause. While a cause covers two notions: sufficiency and necessity, a CS  presents a sufficient cause only. The occurrence of the events in the cut leads to the occurrence of the top-level event. This roughly corresponds to AC1, while the minimality of the cut set corresponds to AC3. The difference lies in the necessity of the cause (AC2). An MCS computation does not include this step, which is the core of actual causality computation. Cut sets are all the enumerations that make the effect true. %Some of these enumerations will be an HP cause, and others will not depending on necessity. % If an HP cause, in a fault-tree causal model, exists it will be one of the minimal cut sets. % Not all the minimal cut sets are HP causes, e.g., cases of overdetermination (conjunctions). 
%Also, FTA does not reason about combination of events as an effect. 
%That said, for tree-based models, FTA and actual causality can complement each other; e.g., the list of MCSs can be used as a set of hypothesized causes for inference, or causal analysis and model checking can be used to construct the fault tree, like in \cite{kuntz2011probabilistic}.
Similarly, model-based diagnosis (MBD) aims to detect faulty components to explain anomalies in system behavior \cite{reiter1987theory}. The model is a set of logical expressions over a set of components. MBD requires a set of observations that correspond to the context $\vec{U}$; using logical inference, MBD outputs a set of hypotheses for how the system differs from its model, i.e., diagnoses. While MBD can be considered as an approach to infer causality, it does not require counterfactuality of the cause. Although MBD uses a notion of intervention (setting some components to abnormal), this is not  counterfactual reasoning. Instead, it is a sufficiency check since MBD uses a behavioral model, i.e., a representation of the correct behavior.  Like FTA,  diagnoses are sufficient causes, but not actual causes. 

\section{Conclusions and Future Work}\label{sec:conc}
According to HP, a set of events ($\vec{X}$) causes an effect ($\varphi$) if (1) both actually happen; (2) changing some values of $\vec{X}$ while fixing a set $\vec{W}$ of the remaining variables at their original value leads to $\varphi$ not happening; and (3) $\vec{X}$ is minimal. The complexity of the general problem has been established elsewhere. We show that when restricting to binary models, the problem of checking or inferring causality can effectively and efficiently be solved as an optimization problem. The problem is not trivial because intuitively, we need to enumerate all sets $\vec{W}$  from condition (2) and need to check minimality for condition (3). We show how to formulate both properties as an optimization problem instead which immediately gives rise to using a solver to determine if a cause satisfies all conditions, or find one that does. For that, we define an objective function that encodes the distance between cause values in the actual and counterfactual worlds. If we now manage to optimize the problem with a smaller cause, then we know that it satisfies condition (2) but is not minimal.  With an additional objective to quantify responsibility, we also formulate inference as an optimization problem. Using  models with  $8000$ variables, which we deem realistic and necessary for automatically inferred causal models, we show that our approaches answer checking queries in  seconds, and inference queries in  minutes. In the future, we plan to explore the extension of the approach to support non-binary models. 
%\paragraph{Acknowledgements}
%This work was supported by the German Research Foundation DFG under grant no. PR1266/4-1, Conflict resolution and causal inference with integrated socio-technical models

 \bibliographystyle{splncs04}
 \bibliography{bibliography}

\appendix
\section{Proofs}\label{sec:ac3_proof}
In this section, we present proof sketches of the theorems in the paper. 
\subsection{Theorem 2 Proof Idea}
Before presenting the proof of \autoref{theorem:ilp}, we present \autoref{lemma:flemma}. 
Recall that Formula $F$ is constructed as follows.
\begin{small}
	\begin{equation}
	\begin{aligned}
	F := \neg\varphi \wedge \bigwedge\limits_{i=1\ldots n} f(U_i=u_i) 	\wedge\bigwedge\limits_{i=1\ldots m, \not\exists j\bullet X_j=V_i} \left(V_i \leftrightarrow F_{V_i} \lor f(V_i=v_i)\right)  \wedge\bigwedge\limits_{i=1\ldots\ell} f(X_i=\neg x_i)
	\end{aligned}
	\label{eq:f}
	\end{equation} 
\end{small}

\begin{lemma}\label{lemma:flemma}
	Formula $F$ constructed in \autoref{eq:f} is satisfiable iff AC2 holds for a given model $M$, context $\vec{u}$, candidate cause $\vec{X}$, and effect $\varphi$. 
\end{lemma}

\begin{proof} The proof consists of two parts.
	\begin{proofpart}
		SAT($F$) $\implies$ AC2, AC2 holds if $F$ is satisfiable 
	\end{proofpart}
	We show this by contradiction. Assume that $F$ is satisfiable and AC2 does not hold. Based on $F$'s truth assignment, $\vec{v'}$, we cluster the variables  into:

	\noindent\textbf{1.} $\vec{X}$: each variable is fixed exactly to the negation of its original value, i.e.,  $X_i=\neg x_i \forall X_i \in \vec{X}$ (recall $\vec{X} \subseteq \vec{V}$).	\textbf{2.} $\vec{W}^*$: variables in this group, if they exist, have equal truth  and original assignments, i.e., $\langle W^*_1,\ldots,W^*_s\rangle$ s.t. $\forall i\forall j\bullet (i\not= j\Rightarrow W^*_i\not= W^*_j) \wedge (W^*_i=V_j\Leftrightarrow v_j'=v_j)$	\textbf{3.} $\vec{Z}$: variables in this group evaluate differently from their original evaluation, i.e., $ \langle Z_1,\ldots,Z_k\rangle$ s.t. $\forall i\forall j\bullet (i\not= j\Rightarrow Z_i\not= Z_j) \wedge (Z_i=V_j\Leftrightarrow v_j'\not= v_j) \wedge (\forall i\not\exists j\bullet Z_i=X_j$).  \\

\noindent Using $\vec{W^*},\vec{Z}$, we re-write $F$ as $F'$ which is also satisfiable. 
\vspace{-2mm}
\begin{align*}
  \resizebox{1.0\hsize}{!}{%
$F' := \neg\varphi \wedge \bigwedge_{i=1\ldots n} f(U_i=u_i) 
\wedge\bigwedge_{i=1\ldots\ell} f(X_i=\neg x_i)
\wedge\bigwedge_{i=1\ldots s}  f(W^*_i=w^*_i)
\wedge\bigwedge_{i=1\ldots k} (Z_i \leftrightarrow F_{Z_i})$
}
\end{align*}
\vspace{-4mm}
%from here eith map to the causality model or go from the other side
%Since $M$ is acyclic, then there is a unique solution of the equations for a given context $\vec{u}$ \cite{halpern2016actual} that can be inferred from the formula.The truth assignment of the formula fixes values in $\vec{X}, \vec{W}, \vec{U}$ and leaves the other variables in $\vec{Z}$ to evaluate according to their equations, evaluating $\varphi$ to its negation.

\noindent Recall that $M$ is acyclic; therefore there is a unique solution to the equations. Let $\Psi$ be the equations in $M$ without the equations that define the variables $\vec{X}$. Let $\Psi_k$ be $\Psi$ without the equations of some variables in a set $\vec{W_k}$. Since AC2 does not hold, $\forall k \bullet \vec{W_k} \subseteq  V\setminus X \Rightarrow (\vec{X}=\vec{\neg x}\wedge\vec{W_k}=\vec{w_k} \wedge \Psi_k\wedge \neg\varphi)$ evaluates to \textit{false}. %, where $\mathcal{P}(V/\vec{X})$ is the power set of variables in $\mathcal{V}$ without variables in $\vec{X}$. 
In case $\vec{W_k} = \vec{W^*}$, the previous unsatisfiable formula is equivalent to the satisfiable $F'$, implying a contradiction.

\begin{proofpart}\label{lemma:f}
	 AC2 $\implies$ SAT($F$); $F$ is satisfiable if AC2 holds 
\end{proofpart}
	Assume that AC2 holds and $F$ is unsatisfiable. Then $\exists \vec{W},\vec{w},\vec{x}'\bullet  (M,\vec{u})\models(\vec{W}=\vec{w})
	\implies (M,\vec{u})\models\bigl[\vec{X}\leftarrow \vec{x}',\vec{W}\leftarrow\vec{w}\bigr]\neg\varphi$. By definition \cite{halpern2015modification}, $(M, \vec{u}) \models [Y_1 \leftarrow y_1.. Y_k \leftarrow y_k]\varphi$ is equivalent to $(M_{Y_1 \leftarrow y_1..Y_k \leftarrow y_k}, \vec{u}) \models \varphi$, i.e., we replace specific equations in $M$ to obtain a new model $M' = M_{Y_1 \leftarrow y_1, ..., Y_k \leftarrow y_k}$. So,we replace the equations of the variables in $\vec{X},\vec{W}$ in $M$ to obtain a new model, $M'$, such that $(M',\vec{u})\models\neg\varphi$. Equations of $\vec{X},\vec{W}$  variables are now of the form  $V_i=v_i$, i.e., each variable is equal to a constant value. Note that $M'$ is only different from $M$ in the equations of $\vec{X},\vec{W}$. Hence, $M'$	is acyclic and has a unique solution for a given $\vec{U}=\vec{u}$. We construct a formula, $F'$ (shown below), that is a conjunction of the variables in sets $X', W', U$ in $M'$. Because of their equations, each variable is represented by a constant, i.e., a positive or a negative literal. Based on the nature of this formula, it is satisfiable with exactly the same truth assignment as the unique solution of $M'$.   

\begin{small}
	\begin{align*}
	F' := \bigwedge_{i=1\ldots n} f(U_i=u_i) 
	\wedge\bigwedge_{i=1\ldots s}  f(W_i'=w_i') 
	\wedge\bigwedge_{i=1\ldots\ell} f(X_i'= x_i')
	\end{align*}
\end{small}

\noindent Now, we add the remaining variables, i.e., $\forall i \bullet V_i\notin(\vec{X}\cup\vec{W})$, as formulas using the $\leftrightarrow$ operator. The overall formula $F''$, is satisfiable because we have an assignment that makes each equivalence relation true.
 \begin{small}
	\begin{align*}
	F'' :=  F' \wedge \bigwedge_{i=1\ldots m, \not\exists j\bullet X_j=V_i, W_j=V_i}  (V_i \leftrightarrow F_{V_i})
	\end{align*}
\end{small}
	\noindent We have $(M',\vec{u})\models\neg\varphi$, which says that the model evaluates $\neg\varphi$ to true with its unique solution (same assignment of $F''$). We add another clause to $F''$ which evaluates to true and keeps the formula satisfiable. That is, $ F''' :=  F'' \wedge \neg\varphi$. Last, we only have to show the relation between ($F$ and  $F'''$). We can rewrite $F$ (shown at the beginning of the proof) such that we remove all disjuncts of the form $(V_i \leftrightarrow F_{V_i})$ for the variables in $\vec{W}$. Similarly, we remove  all disjuncts of the form $f(V_i=v_i)$ for all the variables that are not in $\vec{W}$. According to our assumption, $F$ is still unsatisfiable, since we removed disjunctions from the clauses. Then, we reach a contradiction since $F$ is equivalent to $F'''$ which is satisfiable for the same clauses.
\end{proof}

\repeattheorem{ilp}
\begin{proof}
	\noindent The proof follows form the remark that $G$, which both formulations are based on, is a generalization of $F$ and is satisfiable if the context $\vec{U}$ makes $\varphi$ evaluate to its negation, given that the semantics of the model is expressed using the constraints added, and the cause set, $\vec{X}$, is not constrained to have other values ($\vec{x'}$). We show this in the following: 
	\begin{enumerate}
		\item Recall $G := \neg\varphi \wedge \bigwedge_{i=1\ldots\ell} f(U_i=u_i)\wedge\bigwedge_{i=1\ldots m, \not\exists j\bullet X_j=V_i} (V_i \leftrightarrow F_{V_i} \lor f(V_i=v_i))	\wedge\bigwedge_{i=1\ldots\ell} (X_i\lor\neg X_i)$. Rewrite the formula to abstract the first part as, $G:= G_{base}\wedge\bigwedge_{i=1\ldots n}( X_i\lor\neg X_i)$.  
		\item  Note how  $\vec{X_{(n)}}$ is added to $G$ as $ (X_1 \vee \neg X_1)\wedge(X_2 \vee \neg X_2)\dots (X_n \vee \neg X_n)$. Re-write this big conjunction to its equivalent disjunctive normal from (DNF) i.e., $ (\neg X_1 \wedge \neg X_2 \dots \wedge \neg X_n) \lor(\neg X_1 \wedge \neg X_2 \dots \wedge X_n)\dots\lor(X_1 \wedge X_2 \dots \wedge X_n)$. Assume wlog
		that all the actual values of $\vec{X_{(n)}}$ were \textit{true}, hence to check them in AC2 we need to have their values negated, i.e., $\neg X_i$. Looking at the DNF, we have $2^n$ clauses that list all the possible cases of negating or fixing the elements in $\vec{X}$. Then, we partition $G$ according to the clauses, i.e, $G:=G_1 \lor G_2\dots G_{2^n}$, where $G_1:=G_{base}\wedge (\neg X_1 \wedge \neg X_2 \dots \wedge \neg X_n)$. $G_1$, is formula $F$ for $\vec{X}$, which according to \autoref{lemma:flemma} is satisfiable iff AC2 holds for $\vec{X}$. $G$ holds if any $G_i$ hold.  		
		\item Generally $G_i$, fixes some (possibly none) elements to their original evaluation ($X_i$) and negates some, possibly none ($G_{2^n}$), other elements ($\neg X_i$).  $G_i$ is an $F$ formula (from \autoref{eq:f}) for the negated variables, in a clause, as $\vec{X}$ but with some special fixed variables that are added to $\vec{W}$.
		 Based on \autoref{lemma:flemma}  $G_i$ is satisfiable iff AC2 for a the subset of the causes given that the other part (fixed) of the cause is in $\vec{W}$, holds. Thus $G$ is satisfiable if AC2 holds for any subset of it.		
		\item The transformation from $G$ to an ILP program $P$ is proved to be correct~\cite{li2004satisfiability}. This means that satisfiability of $G$ entails feasibility of $P$. $P$ is then feasible if AC2 holds for the A.) whole $\vec{X}$, B.) parts of $\vec{X}$, or C.) an empty set of causes. Adding the distance constraint to $P$ results in a new program $P'$.  Recall that the distance will be the count of variables $\in \vec{X}$ that have a value $	\vec{x'}$ in the solution of $P'$. The distance should be greater than 0, i.e., case C is treated. By its nature, the ILP solver will  pick the solution set that makes the distance the least. Hence, if  $\vec{X}$ is minimal in fulfilling AC2 it will be picked, i.e., case A. Similarly case B is treated.
		\item Similarly, since $G$ forms the hard clauses of the MaxSat $G_{max}$, 
		then  $G_{max}$ is satisfiable if $G$ is satisfiable. 
		$G_{max}$ is then satisfiable if AC2 holds for the A.) whole $\vec{X}$, B.) parts of $\vec{X}$, or C.) an empty set of causes. 
		We get rid of (C) by adding $K$ clause as a hard clause. As such, 	$G_{max}$ is satisfiable only when for cases A and B.
	\end{enumerate}
	\end{proof}

\subsection{Inference Proofs}
\repeattheorem{why}
% there could be an assignment with empty X but the solver won't pick.
\begin{proof}  	
 The proof follows from the correspondence between formula $G^*$ and $F$. The proof consists of two parts.
	\begin{proofpart}
		SAT($G^*$) $\implies$ $\exists \vec{X}$ such that AC2 holds for $\vec{X}$
	\end{proofpart}
	We show this by contradiction. Assume that $G^*$ is satisfiable and $\not\exists \vec{X}$ such that AC2 holds. 
	\begin{enumerate}
		\item $	G^* :=   \neg\varphi \wedge \bigwedge_{i=1\ldots n} f(U_i=u_i) 
		\wedge\bigwedge_{i=1\ldots m} \left( \left( \left(V_i \leftrightarrow F_{V_i}\right) \land C^1_{i} \right)  \lor \left( \neg \left(V_i \leftrightarrow F_{V_i}\right) \land \neg C^1_{i} \right) \right) \\ 
		\land \left( \left( V_{orig} \land C^2_{i} \right)  \lor \left( \neg V_{orig} \land \neg C^2_{i} \right) \right) $	
		\item For readability let us call $\left(V_i \leftrightarrow F_{V_i}\right)$ as $e_i$. Since $G^*$ is satisfiable, every conjunction $CON_i$: $\left( \left( e_i \land C^1_{i} \right)  \lor \left( \neg e_i \land \neg C^1_{i} \right) \right)  
		\land \left( \left( V_{orig} \land C^2_{i} \right)  \lor \left( \neg V_{orig} \land \neg C^2_{i} \right) \right)$ holds. It is a matter of natural deduction to show that when $CON_i$ holds with values $(C^1_{i} \lor C^2_{i})$ ($01,10,11$)
		it implies $e_i \lor V_{orig}$, that is proving the following proposition $ ( (e_i \land C^1_{i}) \lor (\neg e_i \land  \neg C^1_{i}) ) \land  ( (V_{orig} \land  C^2_{i}) \lor (\neg V_{orig} \land  \neg C^2_{i}) ) )  \land  (C^1_{i} \lor C^2_{i}) \implies e_i \lor V_{orig} $. The only remaining case of  $(C^1_{i}  C^2_{i})$ is $00$. This case, in turn, implies 
		 $\neg e_i \land \neg V_{orig}$. That is the proposition:  
		 $ ((e_i \land C^1_{i}) \lor (\neg e_i \land  \neg C^1_{i})) \land  ((V_{orig} \land  C^2_{i}) \lor (\neg V_{orig} \land  \neg C^2_{i}) )   \land   (\neg C^1_{i} \land \neg C^2_{i}) \implies \neg e_i \land \neg V_{orig} $ can be proved by deduction. Note that there is no guarantee that the $(C^1_{i}  C^2_{i}) = 00$ case always exists (this case is handled by the algorithm). 		 
		 \item For each variable in $M$, adding the implications from above to a formula $Y= \neg\varphi \wedge \bigwedge_{i=1\ldots n} f(U_i=u_i)$ would result in an $F$ formula (from \autoref{eq:f}) for some $\vec{X}$. $Y$ is satisfiable which by \autoref{lemma:flemma} makes AC2 holds for $\vec{X}$. This contradicts with the first assumption.   
    	\end{enumerate}
	
	\begin{proofpart}\label{part1}
				$\exists \vec{X}$ such that AC2 holds for $\vec{X}$ $\implies$  SAT($G^*$)
	\end{proofpart}
		We show this by contradiction, as well. Assume $\exists \vec{X}$ such that AC2 holds, and  that $G^*$ is un-satisfiable. Since AC2 holds then there exists a stisfiable $F$ as constructed in \autoref{eq:f}. Similar to the first part of the proof, since each variable $V_i$ has a satisfiable conjunction in $F$, it implies a conjuntion in $G^*$ (the inverse of the implications in the first part (without $C^1_{i}, C^2_{i}$)). With that $G^*$ is satisfiable. This contradiction proves the second part of the theorem.
\end{proof}

\end{document}